\documentclass{article}

\PassOptionsToPackage{numbers}{natbib}

\usepackage[final,nonatbib]{nips_2017}

\usepackage[utf8]{inputenc} 
\usepackage[T1]{fontenc}    
\usepackage{hyperref}       
\usepackage{url}            
\usepackage{booktabs}       
\usepackage{amsfonts}       
\usepackage{nicefrac}       
\usepackage{microtype}      

\usepackage{lmodern}
\usepackage[dvips]{graphicx}
\usepackage{amsmath,amsfonts}
\usepackage{psfrag}
\usepackage{bm}
\usepackage{subfig}
\usepackage{mathtools}
\usepackage{amsthm}

\newtheorem{theorem}{Theorem}
\newtheorem{lemma}{Lemma}

\newcommand{\given}{\,|\,}

\usepackage{color}
\definecolor{green}{rgb}{0.0, 0.5, 0.0}

\newcommand{\T}{\top}
\newcommand{\double}[1]{\mathbb{#1}}

\newcommand{\bs}{\backslash}

\usepackage[
		backend=bibtex, 
		sorting=none, 
		maxbibnames=99
]{biblatex}
\title{On the Model Shrinkage Effect of\\Gamma Process Edge Partition Models}

\newcommand*{\affaddr}[1]{#1} 
\newcommand*{\affmark}[1][*]{\textsuperscript{#1}}

\author{
		Iku Ohama\affmark[$\star\ddagger$]\hspace{1cm}
		Issei Sato\affmark[$\dagger$]\hspace{1cm}
		Takuya Kida\affmark[$\ddagger$]\hspace{1cm}
		Hiroki Arimura\affmark[$\ddagger$]\\
		\affaddr{\affmark[$\star$]Panasonic Corp., Japan}\hspace{0.2cm}
		\affaddr{\affmark[$\dagger$]The Univ.~of Tokyo, Japan}\hspace{0.2cm}
		\affaddr{\affmark[$\ddagger$]Hokkaido Univ., Japan}\\
		ohama.iku@jp.panasonic.com\hspace{0.3cm}
		sato@k.u-tokyo.ac.jp\hspace{0.3cm}
		\{kida,arim\}@ist.hokudai.ac.jp
}

\begin{document}

\maketitle


\begin{abstract}
		The edge partition model (EPM) is a fundamental 
		Bayesian nonparametric model for extracting an overlapping structure 
		from binary matrix. 
		The EPM adopts a gamma process ($\Gamma$P) prior to 
		automatically shrink the number of active atoms. 
		However, we empirically found 
		that the model shrinkage of the EPM does not 
		typically work appropriately and leads to an overfitted solution. 
		An analysis of the expectation of the EPM's intensity function 
		suggested that the gamma priors for the EPM hyperparameters 
		disturb the model shrinkage effect of the internal $\Gamma$P. 
		In order to ensure that the model shrinkage effect of the EPM 
		works in an appropriate manner, 
		we proposed two novel generative constructions of the EPM: 
		CEPM incorporating constrained gamma priors, 
		and DEPM incorporating Dirichlet priors instead of the gamma priors. 
		Furthermore, 
		all DEPM's model parameters including the infinite atoms of the $\Gamma$P prior 
		could be marginalized out, 
		and thus it was possible to derive a truly infinite DEPM (IDEPM) 
		that can be efficiently inferred using a collapsed Gibbs sampler. 
		We experimentally confirmed that the model shrinkage of the proposed models 
		works well and that the IDEPM indicated state-of-the-art performance 
		in 
		generalization ability, link prediction accuracy, 
		mixing efficiency, and convergence speed. 
\end{abstract}

\section{Introduction}
%
Discovering low-dimensional structure from a binary matrix is 
an important problem in relational data analysis. 
Bayesian nonparametric priors, such as Dirichlet process (DP)~\cite{FERGUSON73} and 
hierarchical Dirichlet process (HDP)~\cite{TEH2006}, 
have been widely applied to construct 
statistical models with an automatic model shrinkage effect~\cite{IRM,MMSBM}. 
Recently, more advanced stochastic processes 
such as the Indian buffet process (IBP)~\cite{Griffiths2005} 
enabled the construction of statistical models 
for discovering overlapping structures~\cite{IBPIRM,ILA}, 
wherein each individual in a data matrix can belong to multiple latent classes. 

\begin{figure}[tb]
		\centering
		\includegraphics[width=0.98\columnwidth,clip]{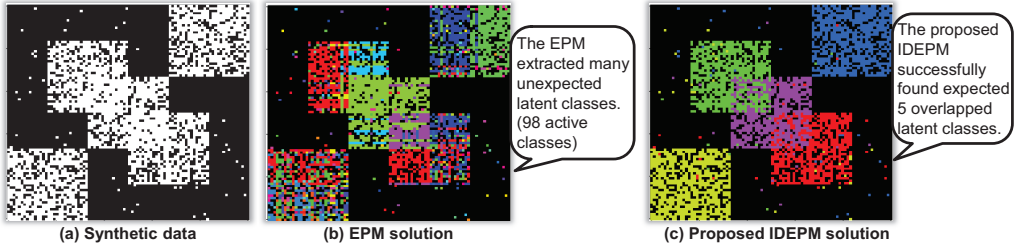}
		\caption{
				(Best viewed in color.) 
				A synthetic example: 
				(a) synthetic $90\times90$ data (white corresponds to one, and black to zero); 
				(b) EPM solution; and (c) the proposed IDEPM solution. 
				In (b) and (c), non-zero entries are colored to indicate their most 
				probable assignment to the latent classes. 
		}
		\label{fig:TOY}
\end{figure}
Among these models, the \emph{edge partition model} (EPM)~\cite{EPM} 
is a fundamental Bayesian nonparametric model for extracting 
overlapping latent structure underlying a given binary matrix. 
The EPM considers latent positive random counts for only non-zero entries 
in a given binary matrix 
and factorizes the count matrix into 
two non-negative matrices and a non-negative diagonal matrix. 
A link probability of the EPM for an entry is defined by 
transforming the multiplication of the non-negative matrices into a probability, 
and thus the EPM can capture overlapping structures with a noisy-OR manner~\cite{IBPIRM}. 
By incorporating a gamma process ($\Gamma$P) as a prior for the diagonal matrix, 
the number of active atoms of the EPM shrinks automatically according to the given data. 
Furthermore, by truncating the infinite atoms of the $\Gamma$P with a finite number, 
all parameters and hyperparameters of the EPM can be inferred using 
closed-form Gibbs sampler. 
Although, the EPM is well designed to capture an overlapping structure 
and has an attractive affinity with a closed-form posterior inference, 
the EPM involves a critical drawback in its model shrinkage mechanism. 
As we experimentally show in Sec.~\ref{sec:exp}, 
we found that the model shrinkage effect of the EPM does not typically work 
in an appropriate manner. 
Figure~\ref{fig:TOY} shows a synthetic example. 
As shown in Fig.~\ref{fig:TOY}a, there are five overlapping latent classes (white blocks). 
However, as shown in Fig.~\ref{fig:TOY}b, the EPM overestimates the number of 
active atoms (classes) and overfits the data.

In this paper, 
we analyze the undesired property of the 
EPM's model shrinkage mechanism 
and 
propose novel generative constructions for the EPM 
to overcome the aforementioned disadvantage. 
%
As shown in Fig.~\ref{fig:TOY}c, 
the IDEPM proposed in this paper successfully shrinks unnecessary atoms. 
More specifically, we have three major contributions in this paper. 

(1) 
We analyse the generative construction of the EPM 
and find a property that disturbs its 
model shrinkage effect (Sec.~\ref{sec:analysis}). 
We derive the expectation of the EPM's intensity function (Theorem~\ref{theorem1}), 
which is the total sum of the infinite atoms for an entry. 
From the derived expectation, 
we obtain a new finding that gamma priors for 
the EPM's hyperparameters 
disturb the model shrinkage effect of the internal $\Gamma$P (Theorem~\ref{theorem11}). 
That is, the derived expectation is expressed by a multiplication of the 
terms related to $\Gamma$P and other gamma priors. 
Thus, there is no guarantee that the expected number of active atoms 
is finite. 

(2) 
Based on the analysis of the EPM's intensity function, 
we propose two novel constructions of the EPM: 
the CEPM incorporating constrained gamma priors (Sec.~\ref{subsec:cepm}) and 
the DEPM incorporating Dirichlet priors 
instead of the gamma priors (Sec.~\ref{subsec:depm}). 
The model shrinkage effect of the CEPM and DEPM works appropriately because 
the expectation of their intensity functions depends only on the $\Gamma$P prior 
(Sec.~\ref{subsec:cepm} and Theorem~\ref{theorem2} in Sec.~\ref{subsec:depm}). 

(3) 
Furthermore, for the DEPM, 
all model parameters, including the infinite atoms of the $\Gamma$P prior, 
can be marginalized out (Theorem~\ref{theorem4}). 
Therefore, we can derive a truly infinite DEPM (IDEPM), 
which has a closed-form marginal likelihood without truncating infinite atoms, 
and can be efficiently inferred using collapsed Gibbs sampler~\cite{CGIBBS} 
(Sec.~\ref{subsec:idepm}). 


\section{The Edge Partition Model (EPM)}
\label{sec:epm}
In this section, 
we review the EPM~\cite{EPM} as a baseline model. 
Let $\bm{x}$ be an $I\times J$ binary matrix, where an entry between 
$i$-th row and $j$-th column is represented by $x_{i,j} \in \{0,1\}$. 
In order to extract an overlapping structure underlying $\bm{x}$, 
the EPM~\cite{EPM} considers 
a non-negative matrix factorization problem on latent Poisson counts as follows: 
\begin{gather}
		\label{eq:epm}
		x_{i,j} = \double{I}(m_{i,j,\cdot}\geq 1), ~~
		m_{i,j,\cdot} \given \bm{U}, \bm{V}, \bm{\lambda} \sim \mbox{Poisson}\left( 
		\sum_{k=1}^K 
		U_{i,k} V_{j,k} \lambda_k 
		\right)
		, 
\end{gather}
where $\bm{U}$ and $\bm{V}$ are $I\times K$ and $J\times K$ non-negative matrices, respectively, 
and $\bm{\lambda}$ is a $K\times K$ non-negative diagonal matrix. 
Note that $\double{I}(\cdot)$ is $1$ if the predicate holds and is zero otherwise. 
The latent counts $\bm{m}$ take positive values only for edges (non-zero entries) within a given binary matrix 
and the generative model for each positive count is equivalently 
expressed as a sum of $K$ Poisson random variables as 
$m_{i,j,\cdot} = \sum_k m_{i,j,k}, m_{i,j,k} \sim \mbox{Poisson}(U_{i,k} V_{j,k} \lambda_k)$. 
This is the reason why the above model is called edge partition model. 
Marginalizing $\bm{m}$ out from Eq.~\eqref{eq:epm}, the generative model of the EPM can be 
equivalently rewritten as 
$
		x_{i,j} \given \bm{U}, \bm{V}, \bm{\lambda} \sim 
		\mbox{Bernoulli}( 
		1-\prod_k e^{-U_{i,k} V_{j,k} \lambda_k}
		)
$. 
As $e^{-U_{i,k} V_{j,k} \lambda_k} \in [0,1]$ denotes the probability that 
a Poisson random variable with mean $U_{i,k} V_{j,k} \lambda_k$ 
corresponds to zero, 
the EPM can capture an overlapping structure with a noisy-OR manner \cite{IBPIRM}. 

In order to complete the Bayesian hierarchical model of the EPM, 
gamma priors are adopted as 
$U_{i,k} \sim \mbox{Gamma}(a_1,b_1)$ and 
$V_{j,k} \sim \mbox{Gamma}(a_2,b_2)$, 
where $a_1,a_2$ are shape parameters and $b_1,b_2$ are rate parameters for the gamma distribution, respectively. 
Furthermore, a gamma process ($\Gamma$P) is incorporated as a Bayesian nonparametric prior for $\bm{\lambda}$ 
to make the EPM automatically shrink its number of atoms $K$. 
Let $\mbox{Gamma}(\gamma_0/T, c_0)$ denote a truncated $\Gamma$P with 
a concentration parameter $\gamma_0$ and a rate parameter $c_0$, where 
$T$ denotes a truncation level that should be set large enough to ensure a 
good approximation to the true $\Gamma$P. 
Then, the diagonal elements of $\bm{\lambda}$ are drawn as $\lambda_k \sim \mbox{Gamma}(\gamma_0/T, c_0)$ for $k \in \{1,\ldots,T\}$. 

The posterior inference for all parameters and hyperparameters of the EPM 
can be performed using Gibbs sampler (detailed in Appendix~\ref{asec:epm}). 
Thanks to the conjugacy between gamma and Poisson distributions, 
given $m_{i,\cdot,k} = \sum_j m_{i,j,k}$ and $m_{\cdot,j,k} = \sum_i m_{i,j,k}$, 
posterior sampling for $U_{i,k}$ and $V_{j,k}$ is straightforward. 
As the $\Gamma$P prior is approximated by a gamma distribution, 
posterior sampling for $\lambda_k$ also can be performed straightforwardly. 
Given $\bm{U}$, $\bm{V}$, and $\bm{\lambda}$, 
posterior sample for $m_{i,j,\cdot}$ can be simulated using 
zero-truncated Poisson (ZTP) distribution~\cite{Geyer2007}. 
Finally, we can obtain sufficient statistics $m_{i,j,k}$ by partitioning 
$m_{i,j,\cdot}$ into $T$ atoms using a multinomial distribution. 
Furthermore, all hyperparameters of the EPM 
(i.e., $\gamma_0$, $c_0$, $a_1$, $a_2$, $b_1$, and $b_2$) 
can also be sampled by assuming a gamma hyper prior $\mbox{Gamma}(e_0,f_0)$. 
Thanks to the conjugacy between gamma distributions, 
posterior sampling for $c_0$, $b_1$, and $b_2$ is straightforward. 
For the remaining hyperparameters, 
we can construct closed-form Gibbs samplers using 
data augmentation techniques~\cite{NEWMAN2009,ESCOBAR1994,TEH2006}.

\section{Analysis for Model Shrinkage Mechanism}
\label{sec:analysis}
%
The EPM is well designed to capture an overlapping structure with a simple Gibbs inference. 
However, the EPM involves a critical drawback in its model shrinkage mechanism. 

For the EPM, a $\Gamma$P prior is incorporated as a prior for 
the non-negative diagonal matrix as $\lambda_k \sim \mbox{Gamma}(\gamma_0/T, c_0)$. 
From the form of the truncated $\Gamma$P, 
thanks to the additive property of independent gamma random variables, 
the total sum of $\lambda_k$ over countably infinite atoms follows 
a gamma distribution as $\sum_{k=1}^\infty \lambda_k \sim \mbox{Gamma}(\gamma_0, c_0)$, 
wherein the intensity function of the $\Gamma$P has a finite expectation 
as $\double{E}[\sum_{k=1}^\infty \lambda_k] = \frac{\gamma_0}{c_0}$. 
Therefore, the $\Gamma$P has a regularization mechanism that automatically 
shrinks the number of atoms according to given observations. 

However, as experimentally shown in Sec.~\ref{sec:exp}, 
the model shrinkage mechanism of the EPM does not work appropriately. 
More specifically, the EPM often overestimates the number of active atoms 
and overfits the data. 
Thus, we analyse the intensity function of the EPM 
to reveal the reason why the model shrinkage mechanism does not work appropriately. 
\begin{theorem}
		\label{theorem1}
		The expectation of the EPM's intensity function $\sum_{k=1}^\infty U_{i,k} V_{j,k} \lambda_k$ 
		for an entry $(i,j)$ is finite and can be expressed as follows: 
		\begin{align}
				\label{eq:theorem1}
				\double{E}\left[ \sum_{k=1}^\infty U_{i,k} V_{j,k} \lambda_k \right] = 
				\frac{a_1}{b_1} \times \frac{a_2}{b_2} \times \frac{\gamma_0}{c_0}. 
		\end{align}
\end{theorem}
\begin{proof}
		As $\bm{U}$, $\bm{V}$, and $\bm{\lambda}$ are independent of each other, 
		the expected value operator is multiplicative for the EPM's intensity function. 
		Using the multiplicativity and the low of total expectation, 
		the proof is completed as 
		$
				\double{E}\left[ \sum_{k=1}^\infty U_{i,k} V_{j,k} \lambda_k \right] = 
				\sum_{k=1}^\infty \double{E}[U_{i,k}] \double{E}[V_{j,k}] \double{E}[\lambda_k] = 
				\frac{a_1}{b_1} \times \frac{a_2}{b_2} \times \double{E}[\sum_{k=1}^\infty \lambda_k]. 
		$
\end{proof}
As Eq.~\eqref{eq:theorem1} in Theorem~\ref{theorem1} shows, 
the expectation of the EPM's intensity function 
is expressed by multiplying individual expectations of 
a $\Gamma$P and two gamma distributions. 
This causes an undesirable property to the model shrinkage effect of the EPM. 
From Theorem~\ref{theorem1}, 
another important theorem about the EPM's model shrinkage effect is obtained as follows: 
\begin{theorem}
		\label{theorem11}
		Given an arbitrary non-negative constant $C$, 
		even if the expectation of the EPM's intensity function in Eq.~\eqref{eq:theorem1} 
		is fixed as 
		$\double{E}\left[ \sum_{k=1}^\infty U_{i,k} V_{j,k} \lambda_k \right] = C$, 
		there exist cases in which 
		the model shrinkage effect of the $\Gamma$P prior disappears. 
\end{theorem}
\begin{proof}
		Substituting 
		$\double{E}\left[ \sum_{k=1}^\infty U_{i,k} V_{j,k} \lambda_k \right] = C$ 
		for Eq.~\eqref{eq:theorem1}, 
		we obtain $C = \frac{a_1}{b_1} \times \frac{a_2}{b_2} \times \frac{\gamma_0}{c_0}$. 
		Since $a_1$, $a_2$, $b_1$, and $b_2$ are gamma random variables, 
		even if the expectation of the EPM's intensity function, $C$, is fixed, 
		$\frac{\gamma_0}{c_0}$ can take an arbitrary value so that 
		equation $C = \frac{a_1}{b_1} \times \frac{a_2}{b_2} \times \frac{\gamma_0}{c_0}$ 
		holds. 
		Hence, $\gamma_0$ can take an arbitrary large value 
		such that $\gamma_0 = T\times \widehat{\gamma}_0$. 
		This implies that the $\Gamma$P prior for the EPM degrades to 
		a gamma distribution without model shrinkage effect as 
		$\lambda_k \sim \mbox{Gamma}(\gamma_0/T, c_0) = \mbox{Gamma}(\widehat{\gamma}_0,c_0)$. 
\end{proof}
Theorem~\ref{theorem11} indicates that 
the EPM might overestimate the number of active atoms, 
and lead to overfitted solutions.


\section{Proposed Generative Constructions}
\label{sec:proposal}
%
We describe our novel generative constructions for the EPM with an appropriate model shrinkage effect. 
According to the analysis described in Sec.~\ref{sec:analysis}, 
the model shrinkage mechanism of the EPM does not work 
because the expectation of the EPM's intensity function has 
an undesirable redundancy. 
This finding motivates the proposal of new generative constructions, 
in which the expectation of the intensity function depends only on the $\Gamma$P prior. 

First, we propose a naive extension of the original EPM 
using constrained gamma priors (termed as CEPM). 
Next, we propose an another generative construction for the EPM 
by incorporating Dirichlet priors instead of gamma priors (termed as DEPM). 
Furthermore, for the DEPM, 
we derive truly infinite DEPM (termed as IDEPM) by marginalizing out all model parameters including the 
infinite atoms of the $\Gamma$P prior.

\subsection{CEPM}
\label{subsec:cepm}
In order to ensure that the EPM's intensity function depends 
solely on the $\Gamma$P prior, 
a naive way is to introduce constraints for the hyperparameters of the gamma prior. 
In the CEPM, the rate parameters of the gamma priors are constrained as 
$b_1 = C_1 \times a_1$ and $b_2 = C_2 \times a_2$, respectively, 
where $C_1 > 0$ and $C_2 > 0$ are arbitrary constants. 
Based on the aforementioned constraints and Theorem~\ref{theorem1}, 
the expectation of the intensity function for the CEPM 
depends only on the $\Gamma$P prior as 
$\double{E}[\sum_{k=1}^\infty U_{i,k} V_{j,k} \lambda_k] = \frac{\gamma_0}{C_1 C_2 c_0}$. 

The posterior inference for the CEPM can be performed using Gibbs sampler 
in a manner similar to that for the EPM. 
However, we can not derive closed-form samplers only for 
$a_1$ and $a_2$ because of the constraints. 
Thus, in this paper, posterior sampling for $a_1$ and $a_2$ are 
performed using grid Gibbs sampling~\cite{Zhou2014} 
(see Appendix~\ref{asec:cepm} for details).

\subsection{DEPM}
\label{subsec:depm}
We have another strategy to construct the EPM with efficient model shrinkage effect 
by re-parametrizing the factorization problem. 
Let us denote transpose of a matrix $\bm{A}$ by $\bm{A}^\top$. 
According to the generative model of the EPM in Eq.~\eqref{eq:epm}, 
the original generative process for counts $\bm{m}$ 
can be viewed as a matrix factorization as 
$\bm{m} \approx \bm{U} \bm{\lambda} \bm{V}^\T$. 
It is clear that the optimal solution of 
the factorization problem is not unique. 
Let $\bm{\Lambda}_1$ and $\bm{\Lambda}_2$ be arbitrary $K\times K$ non-negative diagonal matrices. 
If a solution $\bm{m} \approx \bm{U}\bm{\lambda}\bm{V}^\T$ is globally optimal, 
then another solution 
$\bm{m} \approx (\bm{U}\bm{\Lambda}_1) (\bm{\Lambda}_1^{-1} \bm{\lambda} \bm{\Lambda}_2) (\bm{V}\bm{\Lambda}_2^{-1})^\T$ 
is also optimal. 
In order to ensure that the EPM has only one optimal solution, 
we re-parametrize the original factorization problem to 
an equivalent constrained factorization problem as follows: 
\begin{gather}
		\label{eq:depm}
		\bm{m} \approx \bm{\phi} \bm{\lambda} \bm{\psi}^\T, 
\end{gather}
where $\bm{\phi}$ denotes an $I \times K$ non-negative matrix with $l_1$-constraints as $\sum_i \phi_{i,k} = 1, \forall k$. 
Similarly, $\bm{\psi}$ denotes an $J \times K$ non-negative matrix with $l_1$-constraints as $\sum_j \psi_{j,k} = 1, \forall k$. 
This parameterization ensures the uniqueness of the optimal solution 
for a given $\bm{m}$ because 
each column of $\bm{\phi}$ and $\bm{\psi}$ is constrained 
such that it is defined on a simplex. 

According to the factorization in Eq.~\eqref{eq:depm}, 
by incorporating Dirichlet priors instead of gamma priors, 
the generative construction for $\bm{m}$ of the DEPM is as follows: 
\begin{align}
		\label{eq:gen_depm}
		m_{i,j,\cdot} \given \bm{\phi}, \bm{\psi}, \bm{\lambda} \sim \mbox{Poisson}
		\left(
		\sum_{k=1}^T \phi_{i,k} \psi_{j,k} \lambda_k 
		\right),&~~
		\{\phi_{i,k}\}_{i=1}^I \given \alpha_1 \sim \mbox{Dirichlet}(\overbrace{\alpha_1,\ldots,\alpha_1}^I), \nonumber\\
		\{\psi_{j,k}\}_{j=1}^J \given \alpha_2 \sim \mbox{Dirichlet}(\overbrace{\alpha_2,\ldots,\alpha_2}^J),&~~
		\lambda_k \given \gamma_0, c_0 \sim \mbox{Gamma}(\gamma_0/T, c_0). 
\end{align}
\begin{theorem}
		\label{theorem2}
		The expectation of DEPM's intensity function $\sum_{k=1}^\infty \phi_{i,k} \psi_{j,k} \lambda_k$ 
		depends sorely on the $\Gamma$P prior and can be expressed as 
		$\double{E}[\sum_{k=1}^\infty \phi_{i,k} \psi_{j,k} \lambda_k] = \frac{\gamma_0}{IJ c_0}$. 
\end{theorem}
\begin{proof}
		The expectations of Dirichlet random variables $\phi_{i,k}$ and $\psi_{j,k}$ are $\frac{1}{I}$ and $\frac{1}{J}$, respectively. 
		Similar to the proof for Theorem~\ref{theorem1}, 
		using the multiplicativity of independent random variables and the low of total expectation, 
		the proof is completed as 
		$
				\double{E}\left[ \sum_{k=1}^\infty \phi_{i,k} \psi_{j,k} \lambda_k \right] = 
				\sum_{k=1}^\infty \double{E}[\phi_{i,k}] \double{E}[\psi_{j,k}] \double{E}[\lambda_k] = 
				\frac{1}{I} \times \frac{1}{J} \times \double{E}[\sum_{k=1}^\infty \lambda_k]. 
		$
\end{proof}
Note that, 
if we set constants $C_1=I$ and $C_2=J$ for the CEPM in Sec.~\ref{subsec:cepm}, 
then the expectation of the intensity function for the CEPM is equivalent to 
that for the DEPM in Theorem~\ref{theorem2}. 
Thus, in order to ensure the fairness of comparisons, 
we set $C_1=I$ and $C_2=J$ for the CEPM in the experiments.


As the Gibbs sampler for $\bm{\phi}$ and $\bm{\psi}$ can be 
derived straightforwardly, 
the posterior inference for all parameters and hyperparameters of the DEPM 
also can be performed via closed-form Gibbs sampler (detailed in Appendix~\ref{asec:depm}). 
Differ from the CEPM, 
$l_1$-constraints in the DEPM ensure the uniqueness of its optimal solution. 
Thus, the inference for the DEPM 
is considered as more efficient than that for the CEPM.

\subsection{Truly Infinite DEPM (IDEPM)}
\label{subsec:idepm}
One remarkable property of the DEPM is 
that we can derive a fully marginalized likelihood function. 
%
Similar to the beta-negative binomial topic model~\cite{Zhou2014}, 
we consider a joint distribution for $m_{i,j,\cdot}$ Poisson customers and 
their assignments $\bm{z}_{i,j}=\{z_{i,j,s}\}_{s=1}^{m_{i,j,\cdot}} \in \{1,\cdots,T\}^{m_{i,j,\cdot}}$ to $T$ tables as 
$P(m_{i,j,\cdot}, \bm{z}_{i,j} \given \bm{\phi},\bm{\psi},\bm{\lambda}) = 
P(m_{i,j,\cdot} \given \bm{\phi},\bm{\psi},\bm{\lambda}) \prod_{s=1}^{m_{i,j,\cdot}} P(z_{i,j,s} \given m_{i,j,\cdot}, \bm{\phi},\bm{\psi},\bm{\lambda})$. 
Thanks to the $l_1$-constraints we introduced in Eq.~\eqref{eq:depm}, 
the joint distribution $P(\bm{m},\bm{z}\given \bm{\phi},\bm{\psi},\bm{\lambda})$ has 
a fully factorized form (see Lemma~\ref{lemma1} in Appendix~\ref{asec:pot4}). 
Therefore, marginalizing $\bm{\phi}$, $\bm{\psi}$, and $\bm{\lambda}$ out 
according to the prior construction in Eq.~\eqref{eq:gen_depm}, 
we obtain an analytical marginal likelihood 
$P(\bm{m},\bm{z})$ 
for the truncated DEPM (see Appendix~\ref{asec:pot4} for a detailed derivation).


Furthermore, by taking $T \rightarrow \infty$, 
we can derive a closed-form marginal likelihood 
for the truly infinite version of the DEPM (termed as IDEPM). 
In a manner similar to that in~\cite{Griffiths2011}, 
we consider the likelihood function for partition $[\bm{z}]$ instead of the 
assignments $\bm{z}$. 
Assume we have $K_+$ of $T$ atoms for which $m_{\cdot,\cdot,k} = \sum_i \sum_j m_{i,j,k} > 0$, 
and a partition of $M (=\sum_i \sum_j m_{i,j,\cdot})$ customers into $K_+$ subsets. 
Then, joint marginal likelihood of the IDEPM for 
$[\bm{z}]$ and $\bm{m}$ 
is given by the following theorem, with the proof provided in Appendix~\ref{asec:pot4}: 
\begin{theorem}
		\label{theorem4}
		The marginal likelihood function of the IDEPM is defined as 
		$P(\bm{m},[\bm{z}])_\infty = \lim_{T\rightarrow \infty} P(\bm{m},[\bm{z}]) = 
		\lim_{T\rightarrow \infty} \frac{T!}{(T-K_+)!} P(\bm{m},\bm{z})$, 
		and can be derived as follows: 
		\begin{align}
				\label{eq:depm_mll_infinite}
				P(\bm{m},[\bm{z}]&)_\infty
				=
				\prod_{i=1}^I \prod_{j=1}^J 
				\frac{1}{m_{i,j,\cdot}!} 
				\times 
				\prod_{k=1}^{K_+}
				\frac{
						\Gamma(I \alpha_1)
				}{
						\Gamma(I \alpha_1 + m_{\cdot,\cdot,k})
				}
				\prod_{i=1}^I 
				\frac{
						\Gamma(\alpha_1 + m_{i,\cdot,k})
				}{
						\Gamma(\alpha_1)
				} \nonumber\\
				\times& 
				\prod_{k=1}^{K_+}
				\frac{
						\Gamma(J \alpha_2)
				}{
						\Gamma(J \alpha_2 + m_{\cdot,\cdot,k})
				}
				\prod_{j=1}^J 
				\frac{
						\Gamma(\alpha_2 + m_{\cdot,j,k})
				}{
						\Gamma(\alpha_2)
				}
				\times
				\gamma_0^{K_+}
				\left(
				\frac{c_0}{c_0+1}
				\right)^{\gamma_0}
				\prod_{k=1}^{K_+}
				\frac{
						\Gamma(m_{\cdot,\cdot,k})
				}{
						(c_0+1)^{m_{\cdot,\cdot,k}}
				}, 
		\end{align}
		where $m_{i,\cdot,k} = \sum_j m_{i,j,k}$, $m_{\cdot,j,k} = \sum_i m_{i,j,k}$, and $m_{\cdot,\cdot,k} = \sum_i \sum_j m_{i,j,k}$. Note that $\Gamma(\cdot)$ denotes gamma function. 
\end{theorem}

From Eq.~\eqref{eq:depm_mll_infinite} in Theorem~\ref{theorem4}, 
we can derive collapsed Gibbs sampler~\cite{CGIBBS} 
to perform posterior inference for the IDEPM. 
Since $\bm{\phi}$, $\bm{\psi}$, and $\bm{\lambda}$ have been marginalized out, 
the only latent variables we have to update are $\bm{m}$ and $\bm{z}$. 

\paragraph{Sampling $\bm{z}$:}
Given $\bm{m}$, 
similar to the Chinese restaurant process (CRP)~\cite{CRP1}, 
the posterior probability that $z_{i,j,s}$ is assigned to $k^*$ is given as follows: 
\begin{gather}
		\label{eq:post_z}
		P(z_{i,j,s} = k^* \given \bm{z}_{\bs (ijs)}, \bm{m})
		\propto 
		\left\{
				\begin{array}{ll}
						m_{k^*}^{\bs (ijs)}
						\times
						\frac{
								\alpha_1 + m_{i,\cdot,k^*}^{\bs (ijs)}
						}{
								I\alpha_1 + m_{\cdot,\cdot,k^*}^{\bs (ijs)}
						}
						\times 
						\frac{
								\alpha_2 + m_{\cdot,j,k^*}^{\bs (ijs)}
						}{
								I\alpha_2 + m_{\cdot,\cdot,k^*}^{\bs (ijs)}
						}
						& \mbox{if $m_{\cdot,\cdot,k^*}^{\bs (ijs)} > 0$},\\
						\gamma_0 
						\times 
						\frac{1}{I}
						\times
						\frac{1}{J}
						& \mbox{if $m_{\cdot,\cdot,k^*}^{\bs (ijs)} = 0$},
				\end{array}
				\right.
\end{gather}
where the superscript $\bs (ijs)$ denotes that the corresponding statistics 
are computed excluding the $s$-th customer of entry $(i,j)$. 

\paragraph{Sampling $\bm{m}$:}
Given $\bm{z}$, posteriors for the $\bm{\phi}$ and $\bm{\psi}$ are simulated as 
$\{\phi_{i,k}\}_{i=1}^I \given - \sim \mbox{Dirichlet}(\{\alpha_1+m_{i,\cdot,k}\}_{i=1}^I)$ 
and 
$\{\psi_{j,k}\}_{j=1}^J \given - \sim \mbox{Dirichlet}(\{\alpha_2+m_{\cdot,j,k}\}_{j=1}^J)$ 
for $k \in \{1,\ldots,K_+\}$. 
Furthermore, the posterior sampling of the $\lambda_k$ for $K_+$ active atoms 
can be performed as $\lambda_k \given - \sim \mbox{Gamma}(m_{\cdot,\cdot,k}, c_0 + 1)$. 
Therefore, similar to the sampler for the EPM~\cite{EPM}, 
we can update $\bm{m}$ as follows: 
\begin{gather}
		m_{i,j,\cdot} \given \bm{\phi}, \bm{\psi}, \bm{\lambda} \sim 
		\left\{
				\begin{array}{ll}
						\delta(0) & \mbox{if $x_{i,j}=0$}, \\
						\mbox{ZTP} (\sum_{k=1}^{K_+} \phi_{i,k} \psi_{j,k} \lambda_k) & \mbox{if $x_{i,j}=1$}, 
				\end{array}
		\right. \\
		\{m_{i,j,k}\}_{k=1}^{K_+} \given m_{i,j,\cdot}, \bm{\phi}, \bm{\psi}, \bm{\lambda} \sim \mbox{Multinomial}\left(
		m_{i,j,\cdot};
		\left\{
		\frac{
				\phi_{i,k} \psi_{j,k} \lambda_k
		}{
				\sum_{k'=1}^{K_+} \phi_{i,k'} \psi_{j,k'} \lambda_{k'}
		}
		\right\}_{k=1}^{K_+}
		\right)
		,
\end{gather}
where $\delta(0)$ denotes point mass at zero. 

\paragraph{Sampling hyperparameters:}
We can construct closed-form Gibbs sampler for all hyperparameters of 
the IDEPM 
assuming a gamma prior ($\mbox{Gamma}(e_0, f_0)$). 
Using the additive property of the $\Gamma$P, 
posterior sample for the sum of $\lambda_k$ over unused atoms is obtained as 
$\lambda_{\gamma_0} = \sum_{k' = K_+ +1}^\infty \lambda_{k'} \given - \sim \mbox{Gamma}(\gamma_0, c_0+1)$. 
Consequently, we obtain a closed-form posterior sampler for the rate parameter $c_0$ of the $\Gamma$P as 
$c_0 \given - \sim \mbox{Gamma}(e_0 + \gamma_0, f_0+\lambda_{\gamma_0}+\sum_{k=1}^{K_+} \lambda_k)$. 
For all remaining hyperparameters (i.e., $\alpha_1$, $\alpha_2$, and $\gamma_0$), 
we can derive posterior samplers from Eq.~\eqref{eq:depm_mll_infinite} using 
\emph{data augmentation} 
techniques~\cite{ESCOBAR1994,EPM,TEH2006,NEWMAN2009} (detailed in Appendix~\ref{asec:hyper_idepm}).

\section{Experimental Results}
\label{sec:exp}
%
In previous sections, 
we theoretically analysed the reason why the model shrinkage of the EPM 
does not work appropriately (Sec.~\ref{sec:analysis}) 
and proposed several novel constructions (i.e., CEPM, DEPM, and IDEPM) 
of the EPM with an efficient model shrinkage effect (Sec.~\ref{sec:proposal}). 

The purpose of the experiments 
involves ascertaining the following hypotheses: 
\begin{itemize}
		\item [(H1)] The original EPM overestimates the number of active atoms 
				and overfits the data. 
				In contrast, the model shrinkage mechanisms of the CEPM and DEPM 
				work appropriately. 
				Consequently, the CEPM and DEPM outperform the EPM 
				in generalization ability and link prediction accuracy. 
		\item [(H2)] Compared with the CEPM, the DEPM indicates 
				better generalization ability and link prediction accuracy 
				because of the uniqueness of the DEPM's optimal solution. 
		\item [(H3)] The IDEPM with collapsed Gibbs sampler 
				is superior to the DEPM in 
				generalization ability, link prediction accuracy, 
				mixing efficiency, and convergence speed. 
\end{itemize}

\paragraph{Datasets:}
The first dataset was the Enron~\cite{ENRON} dataset, 
which comprises e-mails sent between 149 Enron employees. 
We extracted e-mail transactions from 
September 2001 and constructed Enron09 dataset. 
For this dataset, $x_{i,j}=1(0)$ was used to indicate whether an e-mail was, or was not, 
sent by the $i$-th employee to the $j$-th employee. 
For larger dataset, we used the MovieLens~\cite{MOVIELENS} dataset, 
which comprises five-point scale ratings of movies submitted by users. 
For this dataset, we set $x_{i,j}=1$ when the rating was higher than three 
and $x_{i,j}=0$ otherwise. 
We prepared two different sized MovieLens dataset: MovieLens100K (943 users and 1,682 movies) 
and MovieLens1M (6,040 users and 3,706 movies). 
The densities of the 
Enron09, MovieLens100K and MovieLens1M datasets were 
0.016, 0.035, and 0.026, respectively. 

\paragraph{Evaluating Measures:}
We adopted three measurements to evaluate the performance of the models. 
The first is the \emph{estimated number of active atoms $K$} for 
evaluating the model shrinkage effect of each model. 
The second is the averaged \emph{Test Data Log Likelihood} (TDLL) for 
evaluating the generalization ability of each model. 
We calculated the averaged likelihood that a test entry takes the actual value. 
For the third measurement, 
as many real-world binary matrices are often sparse, 
we adopted the \emph{Test Data Area Under the Curve of the Precision-Recall curve} (TDAUC-PR)~\cite{Davis2006} 
to evaluate the link prediction ability. 
In order to calculate the TDLL and TDAUC-PR, 
we set all the selected test entries as zero during the inference period, 
because binary observations for unobserved entries are not 
observed as missing values but are observed as zeros in many real-world situations. 

\paragraph{Experimental Settings:}
Posterior inference for the truncated models (i.e., EPM, CEPM, and DEPM) 
were performed using standard (non-collapsed) Gibbs sampler. 
Posterior inference for the IDEPM was performed 
using the collapsed Gibbs sampler derived in Sec.~\ref{subsec:idepm}. 
For all models, we also sampled all hyperparameters assuming the same gamma prior 
($\mbox{Gamma}(e_0,f_0)$). 
For the purpose of fair comparison, 
we set 
hyper-hyperparameters as $e_0 = f_0 = 0.01$ throughout the experiments. 
We ran 600 Gibbs iterations for each model on each dataset and 
used the final 100 iterations 
to calculate the measurements. 
Furthermore, all reported measurements were averaged values obtained by 
10-fold cross validation. 

\paragraph{Results:}
\begin{figure}[tb]
		\centering
		\includegraphics[width=1.0\columnwidth,clip]{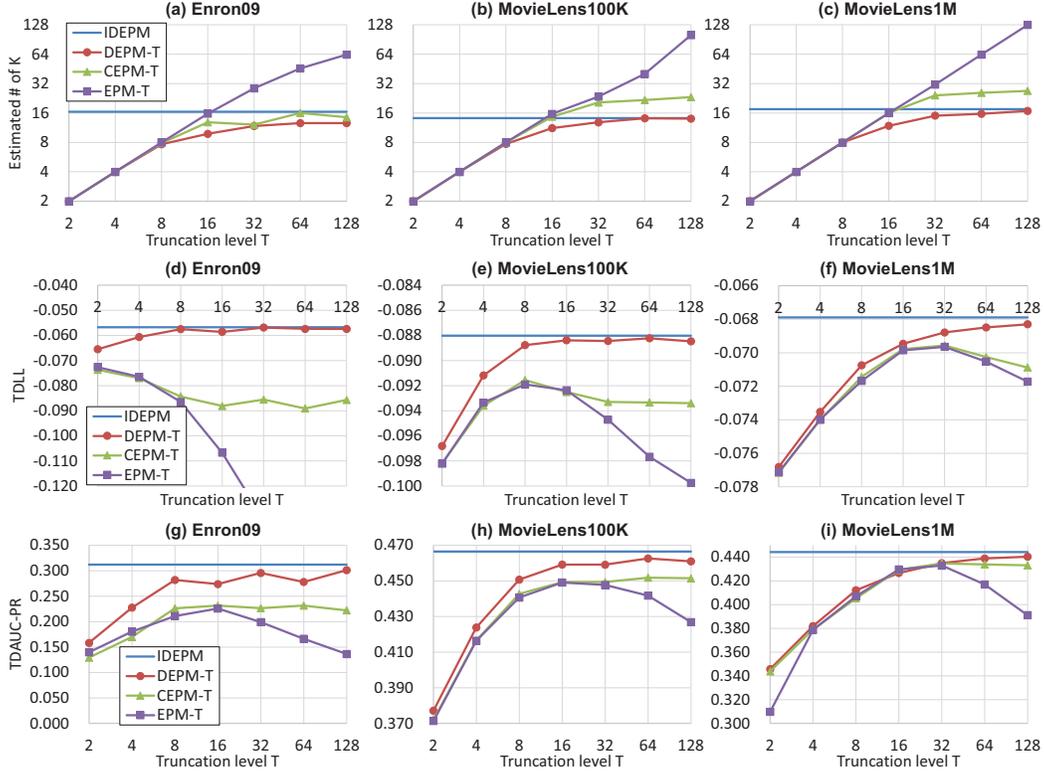}
		\caption{
				Calculated measurements 
				as functions of the truncation level $T$ 
				for each dataset. 
				The horizontal line in each figure denotes the result 
				obtained using the IDEPM. 
		}
		\label{fig:T}
\end{figure}
\begin{figure}[tb]
		\centering
		\includegraphics[width=1.0\columnwidth,clip]{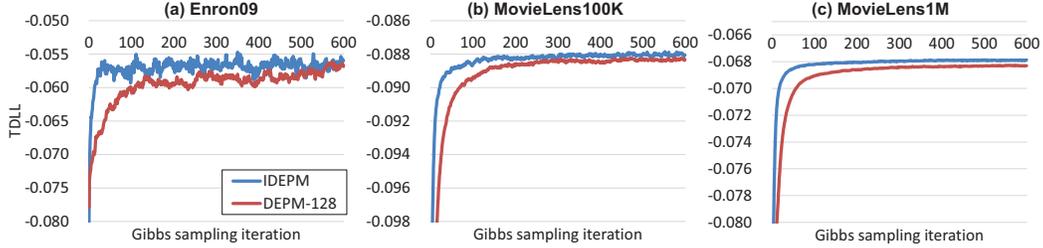}
		\caption{
				(Best viewed in color.) 
				The TDLL as a function of the Gibbs iterations. 
		}
		\label{fig:TDLL-ITR}
\end{figure}
\begin{figure}[tb]
		\centering
		\includegraphics[width=1.0\columnwidth,clip]{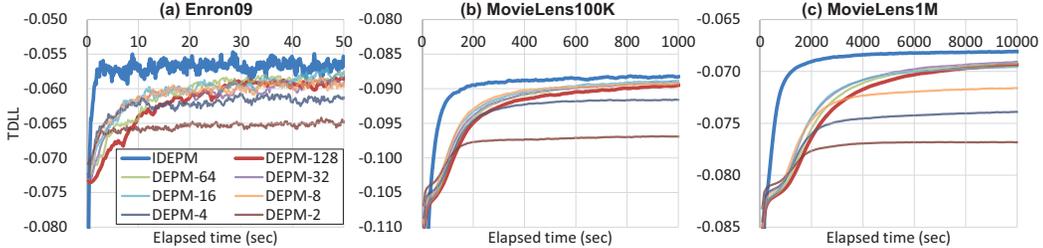}
		\caption{
				(Best viewed in color.) 
				The TDLL as a function of the elapsed time (in seconds). 
		}
		\label{fig:TDLL-TIME}
\end{figure}
Hereafter, 
the truncated models are denoted as 
EPM-$T$, CEPM-$T$, and DEPM-$T$ to specify the truncation level $T$. 
Figure~\ref{fig:T} shows the calculated measurements. 

(H1) 
As shown in Figs.~\ref{fig:T}a--c, 
the EPM overestimated the number of active atoms $K$ for all datasets 
especially for a large truncation level $T$. 
In contrast, 
the number of active atoms $K$ for the 
CEPM-$T$ and DEPM-$T$ monotonically 
converges to a specific value. 
This result supports the analysis with respect to the relationship between 
the model shrinkage effect and the expectation of the EPM's intensity function, 
as discussed in Sec.~\ref{sec:analysis}. 
Consequently, 
as shown by the TDLL (Figs.~\ref{fig:T}d--f) and TDAUC-PR (Figs.~\ref{fig:T}g--i), 
the CEPM and DEPM outperformed the original EPM 
in both generalization ability and link prediction accuracy.

(H2) 
As shown in Figs.~\ref{fig:T}a--c, 
the model shrinkage effect of the DEPM is stronger than that of the CEPM. 
As a result, the DEPM significantly outperformed the CEPM 
in both generalization ability and 
link prediction accuracy (Figs.~\ref{fig:T}d--i). 
Although the CEPM slightly outperformed the EPM, 
the CEPM with a larger $T$ tends to overfit the data. 
In contrast, the DEPM indicated its best performance with 
the largest truncation level ($T=128$). 
Therefore, we confirmed that the uniqueness of the optimal solution 
in the DEPM was considerably important in achieving 
good generalization ability and link prediction accuracy. 


(H3) 
As shown by the horizontal lines in Figs.~\ref{fig:T}d--i, 
the IDEPM 
indicated the state-of-the-art scores for all datasets. 
Finally, 
the computational efficiency 
of the IDEPM was compared with that of the truncated DEPM. 
Figure~\ref{fig:TDLL-ITR} shows the TDLL as a function of the number of Gibbs iterations. 
In keeping with expectations, 
the IDEPM indicated significantly better mixing property when compared with 
that of the DEPM for all datasets. 
Furthermore, Fig.~\ref{fig:TDLL-TIME} shows a comparison of the 
convergence speed of the IDEPM and DEPM with several truncation levels 
($T=\{2,4,8,16,32,64,128\}$). 
As clearly shown in the figure, the convergence of the IDEPM was 
significantly faster than that of the DEPM with all truncation levels. 
Therefore, 
we confirmed that the IDEPM 
indicated a state-of-the-art performance in 
generalization ability, link prediction accuracy, 
mixing efficiency, and convergence speed.

\section{Conclusions}
In this paper, we analysed the model shrinkage effect 
of the EPM, which is a Bayesian nonparametric model for extracting 
overlapping structure 
with an optimal dimension from binary matrices. 
We derived the expectation of the intensity function of the EPM, 
and 
showed that the redundancy of the EPM's intensity function 
disturbs its model shrinkage effect. 
According to this finding, 
we proposed two novel generative construction for the EPM 
(i.e., CEPM and DEPM) to ensure that its model shrinkage effect works appropriately. 
Furthermore, 
we derived a truly infinite version of the DEPM (i.e, IDEPM), 
which can be inferred using 
collapsed Gibbs sampler without any approximation for the $\Gamma$P. 
We experimentally showed that 
the model shrinkage mechanism of the CEPM and DEPM worked appropriately. 
Furthermore, we confirmed that 
the proposed IDEPM indicated a 
state-of-the-art performance in generalization ability, 
link prediction accuracy, mixing efficiency, and convergence speed. 
It is of interest to further investigate whether the 
truly infinite construction of the IDEPM can be applied 
to more complex and modern machine learning models, 
including deep brief networks~\cite{Zhou2015}, 
and tensor factorization models~\cite{Hu2015UAI}.





\appendix

\section{Appendix: Gibbs Samplers for the EPM}
\label{asec:epm}

\subsection{Model Description}
The full description of the generative model for the EPM~\cite{EPM} is described as follows: 
\begin{gather}
		x_{i,j} = \double{I}(m_{i,j,\cdot}\geq 1), ~~
		m_{i,j,\cdot} \given \bm{U}, \bm{V}, \bm{\lambda} \sim \mbox{Poisson}\left( 
		\sum_{k=1}^K 
		U_{i,k} V_{j,k} \lambda_k 
		\right)
		, \nonumber\\
		\label{aeq:epm}
		U_{i,k} \sim \mbox{Gamma}(a_1, b_1), ~~ 
		U_{j,k} \sim \mbox{Gamma}(a_2, b_2), ~~ 
		\lambda_k \sim \mbox{Gamma}(\gamma_0/T, c_0). 
\end{gather}

\subsection{Closed-form Gibbs Samplers}
Posterior inference for all parameters and hyperparameters of the EPM 
can be performed using Gibbs sampler. 

\paragraph{Sampling $\bm{m}$:}
From Eq.~\eqref{aeq:epm}, 
as $m_{i,j,\cdot} = 0$ if and only if $x_{i,j} = 0$, posterior sampling of $\bm{m}$ 
is required only for non-zero entries ($x_{i,j}=1$), 
and can be performed using 
zero-truncated Poisson (ZTP) distribution~\cite{Geyer2007} 
as follows: 
\begin{align}
		\label{aeq:epm_sample_mij}
		m_{i,j,\cdot} \given \bm{U}, \bm{\lambda}, \bm{V} \sim 
		\left\{
				\begin{array}{ll}
						\delta(0) & \mbox{if $x_{i,j}=0$}, \\
						\mbox{ZTP} (\sum_{k=1}^{T} U_{i,k} \lambda_k V_{j,k}) & \mbox{if $x_{i,j}=1$}. 
				\end{array}
		\right.
\end{align}
Then, 
latent count $m_{i,j,k}$ related to the $k$-th atom 
can be obtained by partitioning $m_{i,j,\cdot}$ into $T$ atoms as 
\begin{align}
		\label{aeq:epm_sample_mijk}
		\{m_{i,j,k}\}_{k=1}^T \given m_{i,j,\cdot}, \bm{U}, \bm{\lambda}, \bm{V} \sim \mbox{Multinomial}\left(
		m_{i,j,\cdot};
		\left\{
		\frac{
				U_{i,k} \lambda_k V_{j,k}
		}{
				\sum_{k'=1}^T U_{i,k'} \lambda_{k'} V_{j,k'} 
		}
		\right\}_{k=1}^T
		\right). 
\end{align}

\paragraph{Sampling $\bm{U},\bm{V},\bm{\lambda}$:}
As the generative model for $m_{i,j,k}$ can be given as 
$m_{i,j,k} \given \bm{U}, \bm{V}, \bm{\lambda} \sim \mbox{Poisson}(U_{i,k} V_{j,k} \lambda_k)$, 
according to the additive property of the Poisson distributions, 
generative models for aggregated counts also can be expressed as follows: 
\begin{align}
		\label{aeq:gen_mik}
		m_{i,\cdot,k} = ({\scriptstyle \sum_j} m_{i,j,k}) \given \bm{U}, \bm{V}, \bm{\lambda} 
		\sim \mbox{Poisson}(U_{i,k} ({\scriptstyle \sum_j} V_{j,k}) \lambda_k), \\
		\label{aeq:gen_mjk}
		m_{\cdot,j,k} = ({\scriptstyle \sum_i} m_{i,j,k}) \given \bm{U}, \bm{V}, \bm{\lambda} 
		\sim \mbox{Poisson}(({\scriptstyle \sum_i} U_{i,k}) V_{j,k} \lambda_k), \\
		\label{aeq:gen_lambdak}
		m_{\cdot,\cdot,k} = ({\scriptstyle \sum_i \sum_j} m_{i,j,k}) \given \bm{U}, \bm{V}, \bm{\lambda} 
		\sim \mbox{Poisson}(({\scriptstyle \sum_i} U_{i,k}) ({\scriptstyle \sum_j} V_{j,k}) \lambda_k). 
\end{align}
Therefore, thanks to the conjugacy between Poisson and gamma distributions, 
posterior samplers for $\bm{U}$, $\bm{V}$, and $\bm{\lambda}$ are straightforwardly derived as follows: 
\begin{gather}
		U_{i,k} \given - \sim \mbox{Gamma}(a_1 + m_{i,\cdot,k}, b_1+({\scriptstyle \sum_j} V_{j,k}) \lambda_k), \\
		V_{j,k} \given - \sim \mbox{Gamma}(a_2 + m_{\cdot,j,k}, b_2+({\scriptstyle \sum_i} U_{i,k}) \lambda_k), \\
		\label{aeq:epm_sample_lambdak}
		\lambda_k \given - \sim \mbox{Gamma}(\gamma_0/T + m_{\cdot,\cdot,k}, c_0+({\scriptstyle \sum_i} U_{i,k}) ({\scriptstyle \sum_j} V_{j,k})). 
\end{gather}

\subsection{Sampling Hyperparameters}
\label{asubsec:hyper}
\paragraph{Sampling $b_1,b_2,c_0$:}
Thanks to the conjugacy between gamma distributions, 
posterior samplers for $b_1$, $b_2$, and $c_0$ are straightforwardly performed as follows: 
\begin{align}
		b_1 \given - &\sim \mbox{Gamma}(e_0 + I T a_1, f_0 + {\scriptstyle \sum_i} {\scriptstyle \sum_k} \phi_{i,k}), \\
		b_2 \given - &\sim \mbox{Gamma}(e_0 + J T a_2, f_0 + {\scriptstyle \sum_j} {\scriptstyle \sum_k} \psi_{j,k}), \\
		\label{aeq:epm_sampler_c0}
		c_0 \given - &\sim \mbox{Gamma}(e_0 + \gamma_0, f_0 + {\scriptstyle \sum_k} \lambda_k). 
\end{align}

For the remaining hyperparameters (i.e., $a_1$, $a_2$, and $\gamma_0$), 
we can construct closed-form Gibbs samplers using data augmentation techniques~\cite{ESCOBAR1994,EPM,TEH2006,NEWMAN2009}, 
that consider an expanded probability over target and some auxiliary variables. 
The key strategy is the use of the following expansions: 
\begin{align}
		\label{eq:expansion1}
		\frac{\Gamma(u)}{\Gamma(u+n)} =& 
		\frac{B(u,n)}{\Gamma(n)} = 
		\Gamma(n)^{-1}\int_0^1 v^{u-1} (1-v)^{n-1} \mathrm{d}v ,\\ 
		\label{eq:expansion2}
		\frac{\Gamma(u+n)}{\Gamma(u)} =& \sum_{w=0}^{n} S(n,w) u^w , 
\end{align}
where $B(\cdot,\cdot)$ is the beta function and 
$S(\cdot,\cdot)$ is the Stirling number of the first kind. 

\paragraph{Sampling $a_1,a_2$:}
For shape parameter $a_1$, 
marginalizing $\bm{U}$ from Eq.~\eqref{aeq:gen_mik}, 
we have a partially marginalized likelihood related to target variable $a_1$ as: 
\begin{align}
		\label{aeq:pmll_mik}
		P(\{ m_{i,\cdot,k} \}_{i,k} \given \bm{V}, \bm{\lambda}) &\propto 
		\prod_{k=1}^T 
		\left\{
		\left( 
		\frac{b_1}{b_1+({\scriptstyle \sum_j} V_{j,k}) \lambda_k}
		\right)^{I a_1} 
		\prod_{i=1}^I 
		\frac{\Gamma(a_1 + m_{i,\cdot,k})}{\Gamma(a_1)}
		\right\}
		.
\end{align}
Therefore, expanding Eq.~\eqref{aeq:pmll_mik} using Eq.~\eqref{eq:expansion2} 
and assuming gamma prior as $a_1 \sim \mbox{Gamma}(e_0,f_0)$, 
posterior sampling for $a_1$ can be performed as follows: 
\begin{gather}
		w_{i,k} \given - \sim \mbox{Antoniak}(m_{i,\cdot,k}, a_1), \\
		a_1 \given - \sim 
		\mbox{Gamma}\left(e_0 + {\scriptstyle \sum_i}{\scriptstyle \sum_k} w_{i,k}, 
		f_0 - I \times {\scriptstyle \sum_k} \ln \frac{b_1}{b_1 + ({\scriptstyle \sum_j} V_{j,k})} 
		\right), 
\end{gather}
where $\mbox{Antoniak}(m_{i,\cdot,k}, a_1)$ is an Antoniak distribution~\cite{ANTONIAK}. 
This is the distribution of the number of occupied tables if $m_{i,\cdot,k}$ customers are 
assigned to one of an infinite number of tables using the Chinese restaurant process (CRP)~\cite{CRP1,CRP2} 
with concentration parameter $a_1$, 
and is sampled as 
$w_{i,k} = \sum_{p=1}^{m_{i,\cdot,k}} w_{i,k,p}, w_{i,k,p} \sim \mbox{Bernoulli}\left(\frac{a_1}{a_1+p-1}\right)$. 
Similarly, posterior sampler for $a_2$ can be derived from 
Eqs.~\eqref{aeq:gen_mjk} and \eqref{eq:expansion2} (omitted for brevity). 

\paragraph{Sampling $\gamma_0$:}
Similar to the samplers for $a_1$ and $a_2$, 
according to Eqs.~\eqref{aeq:gen_lambdak} and \eqref{eq:expansion2}, 
$\gamma_0$ can be updated as follows: 
\begin{gather}
		\label{aeq:epm_sampler_gamma0_1}
		w_k \given - \sim \mbox{Antoniak}(m_{\cdot,\cdot,k}, \gamma_0/T), \\
		\label{aeq:epm_sampler_gamma0_2}
		\gamma_0 \given - \sim \mbox{Gamma}\left(
		e_0 + {\scriptstyle \sum_k} w_k, 
		f_0 - \frac{1}{T} 
		{\scriptstyle \sum_k} \ln 
		\frac{c_0}{c_0+({\scriptstyle \sum_i} U_{i,k})({\scriptstyle \sum_j} V_{j,k})}
		\right). 
\end{gather}

\section{Appendix: Gibbs Samplers for the CEPM}
\label{asec:cepm}
Posterior inference for the CEPM can be performed using Gibbs sampler as same as that for the EPM. 
However, only $a_1$ and $a_2$ do not have closed-form sampler because of introduced constraints 
$b_1 = C_1 \times a_1$ and $b_2 = C_2 \times a_2$. 
Therefore, instead of sampling from true posterior, 
we use the grid Gibbs sampler~\cite{Zhou2014} to sample from a discrete probability distribution 
\begin{gather}
		P(a_1 \given -) \propto \mbox{Eq~\eqref{aeq:pmll_mik}} \times P(a_1) 
\end{gather}
over a grid of points $\frac{1}{1+a_1}=0.01,0.02,\ldots,0.99$. 
Note that $a_2$ can be sampled in a same way as $a_1$ (omitted for brevity).

\section{Appendix: Gibbs Samplers for the DEPM}
\label{asec:depm}

\subsection{Closed-form Gibbs Samplers}
\paragraph{Sampling $\bm{\phi}, \bm{\psi}$:}
Given $m_{\cdot,\cdot,k} = \sum_i \sum_j m_{i,j,k}$, 
generative process for latent count $m_{i,\cdot,k}$ can be expressed as 
\begin{gather}
		\label{aeq:gen_mik_depm}
		\{m_{i,\cdot,k}\}_{i=1}^I \given m_{\cdot,\cdot,k}, \bm{\phi}, \bm{\psi}, \bm{\lambda} 
		\sim 
		\mbox{Multinomial}
		\left(
		m_{\cdot,\cdot,k}; 
		\{
				\phi_{i,k} 
		\}_{i=1}^I
		\right)
		.
\end{gather}
Thanks to conjugacy between Eq.~\eqref{aeq:gen_mik_depm} and Dirichlet prior in Eq.~\eqref{eq:gen_depm}, 
posterior sampling for $\bm{\phi}$ can be performed as 
\begin{gather}
		\{ \phi_{i,k} \}_{i=1}^I \given - \sim \mbox{Dirichlet}(\{\alpha_1 + m_{i,\cdot,k}\}_{i=1}^I). 
\end{gather}
Similarly, $\bm{\psi}$ can be updated as 
\begin{gather}
		\{ \psi_{j,k} \}_{j=1}^J \given - \sim \mbox{Dirichlet}(\{\alpha_2 + m_{\cdot,j,k}\}_{j=1}^J). 
\end{gather}

\paragraph{Sampling $\bm{m}, \bm{\lambda}$:}
Posterior samplers for remaining latent variables $\bm{m}$ and $\bm{\lambda}$ are 
straightforwardly given from Eqs.~\eqref{aeq:epm_sample_mij}, \eqref{aeq:epm_sample_mijk}, and 
\eqref{aeq:epm_sample_lambdak} by replacing $\bm{U}$ and $\bm{V}$ with $\bm{\phi}$ and $\bm{\psi}$, respectively.

\subsection{Sampling Hyperparameters}
\label{asubsec:depm_hyper}
\paragraph{Sampling $\alpha_1, \alpha_2$:}
Similar to Appendix~\ref{asubsec:hyper}, 
marginalizing $\bm{\phi}$ out from Eq.~\eqref{eq:gen_depm} and 
expanding the marginal likelihood using Eqs.~\eqref{eq:expansion1} and \eqref{eq:expansion2}, 
posterior sampling for $\alpha_1$ can be derived as follows: 
\begin{gather}
		v_{1,k} \given - \sim \mbox{Beta}(I \alpha_1, m_{\cdot,\cdot,k}), \\ 
		w_{1,i,k} \given - \sim \mbox{Antoniak}(m_{i,\cdot,k}, \alpha_1), \\
		\alpha_1 \given - \sim \mbox{Gamma}(e_0+{\scriptstyle \sum_i}{\scriptstyle \sum_k} w_{1,i,k}, 
		f_0-I\times {\scriptstyle \sum_k} \ln v_{1,k}). 
\end{gather}
Note that the posterior sampler for $\alpha_2$ can be derived in same way (omitted for brevity). 

\paragraph{Sampling $\gamma_0, c_0$:}
The remaining hyperparameters (i.e., $\gamma_0$ and $c_0$) 
can be updated as same as in the EPM. 
Similar to the sampler for the EPM, 
$c_0$ can be updated using Eq.~\eqref{aeq:epm_sampler_c0}. 
Finally, posterior sampler for $\gamma_0$ can be derived as 
\begin{gather}
		\label{aeq:depm_sampler_gamma0_1}
		w_k \given - \sim \mbox{Antoniak}(m_{\cdot,\cdot,k}, \gamma_0/T), \\
		\label{aeq:depm_sampler_gamma0_2}
		\gamma_0 \given - \sim \mbox{Gamma}\left(
		e_0 + {\scriptstyle \sum_k} w_k, 
		f_0 - \ln 
		\frac{c_0}{c_0+1}
		\right). 
\end{gather}

\section{Appendix: Proof of Theorem~\ref{theorem4}}
\label{asec:pot4}
Considering a joint distribution for $m_{i,j,\cdot}$ customers and 
their assignments $\bm{z}_{i,j}=\{z_{i,j,s}\}_{s=1}^{m_{i,j,\cdot}} \in \{1,\cdots,T\}^{m_{i,j,\cdot}}$ to $T$ tables, 
we have following lemma for the truncated DEPM: 
\begin{lemma}
		\label{lemma1}
		The joint distribution over $\bm{m}$ and $\bm{z}$ for the DEPM is expressed by a fully 
		factorized form as 
		\begin{align}
				\label{eq:depm_ll}
				P(\bm{m}, \bm{z} \given \bm{\phi}, \bm{\psi}, \bm{\lambda})
				&=
				\prod_{i=1}^I \prod_{j=1}^J
				\frac{1}{m_{i,j,\cdot}!} 
				\times 
				\prod_{i=1}^I \prod_{k=1}^T 
				\phi_{i,k}^{m_{i,\cdot,k}}
				\times
				\prod_{j=1}^J \prod_{k=1}^T 
				\psi_{j,k}^{m_{\cdot,j,k}}
				\times
				\prod_{k=1}^T
				\lambda_k^{m_{\cdot,\cdot,k}}
				e^{-\lambda_k}. 
		\end{align}
\end{lemma}
\begin{proof}
		As the likelihood functions $P(m_{i,j,\cdot}\given \bm{\phi}, \bm{\psi}, \bm{\lambda})$ and 
		$P(z_{i,j,s}\given m_{i,j,\cdot}, \bm{\phi}, \bm{\psi}, \bm{\lambda})$ are given as 
		\begin{gather}
				P(m_{i,j,\cdot} \given \bm{\phi},\bm{\psi},\bm{\lambda}) = 
				\frac{1}{m_{i,j,\cdot}!} 
				\left(\sum_{k=1}^T \phi_{i,k} \psi_{j,k} \lambda_k \right)^{m_{i,j,\cdot}}
				e^{-\sum_{k=1}^T \phi_{i,k} \psi_{j,k} \lambda_k}, \\
				P(z_{i,j,s}=k^* \given m_{i,j,\cdot}, \bm{\phi},\bm{\psi},\bm{\lambda}) = 
				\frac{
						\phi_{i,k^*} \psi_{j,k^*} \lambda_{k^*}
				}{
						\sum_{k'=1}^T \phi_{i,k'} \psi_{j,k'} \lambda_{k'}
				}, 
		\end{gather}
		respectively, we obtain the joint likelihood function for $\bm{m}$ and $\bm{z}$ as follows: 
		\begin{align}
				\label{eq:depm_ll_tmp}
				\MoveEqLeft P(\bm{m}, \bm{z} \given \bm{\phi}, \bm{\psi}, \bm{\lambda}) \nonumber\\
				=& 
				\prod_{i=1}^I \prod_{j=1}^J 
				\left\{
						P(m_{i,j,\cdot} \given \bm{\phi}, \bm{\psi}, \bm{\lambda}) 
						\prod_{s=1}^{m_{i,j,\cdot}} P(z_{i,j,s} \given m_{i,j,\cdot}, \bm{\phi},\bm{\psi},\bm{\lambda})
				\right\}
				\nonumber\\
				=&
				\prod_{i=1}^I \prod_{j=1}^J \frac{1}{m_{i,j,\cdot}!} 
				\times 
				\prod_{i=1}^I \prod_{k=1}^T 
				\phi_{i,k}^{m_{i,\cdot,k}}
				\times
				\prod_{j=1}^J \prod_{k=1}^T 
				\psi_{j,k}^{m_{\cdot,j,k}} 
				\times
				\prod_{k=1}^T
				\lambda_k^{m_{\cdot,\cdot,k}}
				e^{-\lambda_k ({\scriptstyle \sum_i} \phi_{i,k}) ({\scriptstyle \sum_j} \psi_{j,k})}. 
		\end{align}
		Thanks to the $l_1$-constraints for $\bm{\phi}$ and $\bm{\psi}$ we introduced in Eq.~\eqref{eq:depm}, 
		substituting ${\scriptstyle \sum_i} \phi_{i,k} = {\scriptstyle \sum_j} \psi_{j,k} = 1$ for Eq.~\eqref{eq:depm_ll_tmp}, 
		we obtain Eq.~\eqref{eq:depm_ll} in Lemma~\ref{lemma1}. 
\end{proof}

Thanks to the conjugacy between Eq.~\eqref{eq:depm_ll} in Lemma~\ref{lemma1} and 
prior construction in Eq.~\eqref{eq:gen_depm}, 
marginalizing $\bm{\phi}$, $\bm{\psi}$, and $\bm{\lambda}$ out, 
we obtain the following marginal likelihood for the DEPM: 
\begin{align}
		\label{eq:depm_mll}
		P(\bm{m},\bm{z})
		=&
		\prod_{i=1}^I \prod_{j=1}^J 
		\frac{1}{m_{i,j,\cdot}!} 
		\times 
		\prod_{k=1}^T
		\frac{
				\Gamma(I \alpha_1)
		}{
				\Gamma(I \alpha_1 + m_{\cdot,\cdot,k})
		}
		\prod_{i=1}^I 
		\frac{
				\Gamma(\alpha_1 + m_{i,\cdot,k})
		}{
				\Gamma(\alpha_1) 
		} 
		\nonumber\\
		&
		\times 
		\prod_{k=1}^T
		\frac{
				\Gamma(J \alpha_2)
		}{
				\Gamma(J \alpha_2 + m_{\cdot,\cdot,k})
		}
		\prod_{j=1}^J 
		\frac{
				\Gamma(\alpha_2 + m_{\cdot,j,k})
		}{
				\Gamma(\alpha_2) 
		}
		\times
		\prod_{k=1}^T
		\frac{
				\Gamma\left(
				\frac{\gamma_0}{T} + m_{\cdot,\cdot,k}
				\right)
				c_0^{\frac{\gamma_0}{T}}
		}{
				\Gamma\left(
				\frac{\gamma_0}{T}
				\right)
				(c_0 + 1)^{\frac{\gamma_0}{T}+m_{\cdot,\cdot,k}}
		}. 
\end{align}

Considering a partition $[\bm{z}]$ instead of the assignments $\bm{z}$ as same as in~\cite{Griffiths2011}, 
the marginal likelihood function $P(\bm{m},[\bm{z}])$ for a partition of the truncated DEPM can be expressed as 
\begin{align}
		\label{eq:depm_mll_partition}
		P(\bm{m},[\bm{z}]) =& 
		\frac{T!}{(T-K_+)!} P(\bm{m},\bm{z}) \nonumber\\ 
		=& 
		\prod_{i=1}^I \prod_{j=1}^J 
		\frac{1}{m_{i,j,\cdot}!} 
		\times 
		\prod_{k=1}^{K_+}
		\frac{
				\Gamma(I \alpha_1)
		}{
				\Gamma(I \alpha_1 + m_{\cdot,\cdot,k})
		}
		\prod_{i=1}^I 
		\frac{
				\Gamma(\alpha_1 + m_{i,\cdot,k})
		}{
				\Gamma(\alpha_1) 
		} 
		\nonumber\\
		&
		\times 
		\prod_{k=1}^{K_+}
		\frac{
				\Gamma(J \alpha_2)
		}{
				\Gamma(J \alpha_2 + m_{\cdot,\cdot,k})
		}
		\prod_{j=1}^J 
		\frac{
				\Gamma(\alpha_2 + m_{\cdot,j,k})
		}{
				\Gamma(\alpha_2) 
		} \nonumber\\
		&
		\times 
		\frac{T!}{(T-K_+)! T^{K_+}} 
		\times 
		\gamma_0^{K_+} \left( \frac{c_0}{c_0+1} \right)^{\gamma_0} 
		\prod_{k=1}^{K_+} 
		\frac{ \prod_{l=1}^{m_{\cdot,\cdot,k} -1} (l+\gamma_0/T)}{(c_0+1)^{m_{\cdot,\cdot,k}}}. 
\end{align}
Therefore, taking $T\rightarrow \infty$ in Eq.~\eqref{eq:depm_mll_partition}, 
we obtain the marginal likelihood function for the truly infinite DEPM (i.e., IDEPM) 
as in Eq.~\eqref{eq:depm_mll_infinite} of Theorem~\ref{theorem4}.

\section{Appendix: Sampling Hyperparameters for the IDEPM}
\label{asec:hyper_idepm}

\paragraph{Sampling $\alpha_1, \alpha_2$:}
Posterior samplers for $\alpha_1$ and $\alpha_2$ of the IDEPM are equivalent to those of the truncated DEPM 
as in Appendix~\ref{asubsec:depm_hyper}.

\paragraph{Sampling $\gamma_0$:}
From Eq.~\eqref{eq:depm_mll_infinite}, 
we straightforwardly obtain the posterior sampler for $\gamma_0$ as 
\begin{gather}
		\gamma_0 \given - \sim \mbox{Gamma}\left( e_0 + K_+, f_0 - \ln \frac{c_0}{c_0+1} \right). 
\end{gather}
Note that $\gamma_0$ in Eq.~\eqref{eq:depm_mll_infinite} can be marginalized out assuming gamma prior. 
However, we explicitly sample $\gamma_0$ for simplicity in this paper. 

\paragraph{Sampling $c_0$:}
As derived in Sec.~\ref{subsec:idepm} of main article, 
$c_0$ is updated as 
\begin{gather}
		\lambda_k \given - \sim \mbox{Gamma}(m_{\cdot,\cdot,k}, c_0+1)~~~~k \in \{1,\ldots,K_+\}, \\
		\lambda_{\gamma_0}  \given - \sim \mbox{Gamma}(\gamma_0, c_0+1), \\
		c_0 \given - \sim \mbox{Gamma}(e_0 + \gamma_0, f_0 + \lambda_{\gamma_0}+ {\scriptstyle \sum_{k=1}^{K_+}} \lambda_k). 
\end{gather}

\printbibliography

\end{document}